\documentclass[10pt]{article}

\usepackage[letterpaper, margin=1in]{geometry}
\usepackage[letterpaper]{geometry}
\usepackage{algorithm}
\usepackage{algorithmic}
\usepackage{caption}
\usepackage{subcaption}
\usepackage{authblk}
\usepackage{setspace}
\usepackage{ourStyle}
\usepackage[inline]{enumitem}
\usepackage{amsmath,amssymb,graphicx,epstopdf,textpos,rotating}

\usepackage[parfill]{parskip}

\allowdisplaybreaks

\usepackage{fancyhdr}

\fancypagestyle{firststyle}
{
   \fancyhf[L]{\footnotesize In Proceedings of 2016 IEEE International Conference on Data Mining (ICDM), Barcelona.}
   \fancyfoot{}
}

\title{Efficient Distributed SGD with Variance Reduction}
\author[S. De et al.]
       {Soham De and Tom Goldstein\\
       Department of Computer Science, University of Maryland, College Park, USA\\
       \texttt{\{sohamde, tomg\}@cs.umd.ed}}

\begin{document}
\maketitle


\begin{abstract}
Stochastic Gradient Descent (SGD) has become one of the most popular optimization methods for training machine learning models on massive datasets. However, SGD suffers from two main drawbacks: 
\begin{enumerate*}[label=(\roman*)]
  \item The noisy gradient updates have high variance, which slows down convergence as the iterates approach the optimum, and
  \item SGD scales poorly in distributed settings, typically experiencing rapidly decreasing marginal benefits as the number of workers increases.
\end{enumerate*}
In this paper, we propose a highly parallel method, CentralVR, that uses error corrections to reduce the variance of SGD gradient updates, and scales linearly with the number of worker nodes. CentralVR enjoys low iteration complexity, provably \emph{linear} convergence rates, and exhibits linear performance gains up to hundreds of cores for massive datasets. We compare CentralVR to state-of-the-art parallel stochastic optimization methods on a variety of models and datasets, and find that our proposed methods exhibit stronger scaling than other SGD variants.
\end{abstract}

\thispagestyle{firststyle}

\section{Introduction}

Stochastic gradient descent (SGD) \cite{robbins1951stochastic} is a general approach for solving minimization problems where the objective function decomposes into a sum over many terms.  Such a function has the form
\begin{equation}
\label{eq:objective}
\min_x f(x), \hspace{2mm} f(x) = \frac{1}{n}\sum_{i=1}^n f_i(x),
\end{equation}
where each $f_i: \mathbb{R}^d \rightarrow \mathbb{R}.$   This encompasses a
wide range of problem types, including matrix completion and graph cuts~\cite{recht2011hogwild}.  However, the most popular use of SGD is for problems
where the summation in \eqref{eq:objective} is over a large number of elements in a dataset.
For example, each $f_i(x)$ could measure how well a certain model with parameter vector $x$ explains or classifies the $i$-th entry in a large dataset. 

For problems where each term in \eqref{eq:objective} corresponds to a single data observation, SGD selects
a single data index $i_k$ on each iteration $k$, approximates the gradient of the
objective as $ g^k  = \nabla f_{i_k}(x^k) \approx  \nabla f(x^k),$ and then performs the
approximate gradient update
$$x^{k+1} = x^k - \eta g^k.$$
Typically, $i_k$ is chosen uniformly at random from $\{1, 2, \cdots, n\}$ on each iteration $k$, thus making the gradient approximation unbiased. More precisely, we have $\mathbb{E} \big[\nabla f_{i_k}(x^k) ] =  \nabla f(x^k),$ where the expectation is with respect to the random index $i_k$ of the sampled data.
These approximate gradient updates are much cheaper than true gradient update
steps, which is highly advantageous when $x^k$ is far from the true solution.

A major drawback of SGD is that it is an inherently sequential
algorithm. For truly large datasets, parallel or distributed algorithms are
vital, driving interest in SGD variants that parallelize over massive
distributed datasets. While there has been quite a bit of
recent work in the area of parallel asynchronous SGD algorithms
\cite{recht2011hogwild, dean2012large, lian2015asynchronous, agarwal2011distributed,
li2014communication, shamir2014distributed, zinkevich2009slow,
bertsekas1989parallel, zinkevich2010parallelized, zhang2015deep}, these
methods typically experience substantially reduced marginal benefit as the
number of worker nodes increase over a certain limit. Thus, while some of
these algorithms scale linearly when the number of worker nodes is small,
they are less effective when the data is distributed over hundreds or
thousands of nodes.

Moreover, most research in parallel or distributed SGD methods has been focused on the
parameter server model of computation \cite{recht2011hogwild, dean2012large, agarwal2011distributed, li2014communication, zinkevich2009slow}, where each
update to the centrally stored parameter vector requires a communication phase
between the local node and the central server. However, SGD methods tend to
become unstable with infrequent communication, and there has been less work in
the truly distributed setting where communication costs are high
\cite{zinkevich2010parallelized, zhang2015deep, mokhtari2016dsa}.  In this paper, we propose to boost the scalability of
stochastic optimization algorithms using {\em variance reduction} techniques, yielding
SGD methods that scale linearly over hundreds or thousands of nodes and can train models on massive datasets
without the slowdown that existing stochastic methods experience.


\subsection{Background}
%
%

Recently there has been a lot of interest in methods that can control the gradient noise in SGD. These approaches mitigate the effects of vanishing step sizes, and yield methods that enjoy faster convergence rates both theoretically and empirically. One approach to control the gradient noise is to grow the batch size over time \cite{byrd2012sample, de2017automated, friedlander2012hybrid}. Another conceptually related approach is importance sampling, i.e., choosing training points such that the variance in the gradient estimates is reduced \cite{bouchard2015accelerating, csiba2016importance, needell2014stochastic}. Recently, variance reduction methods have also gained popularity as an alternative to classical SGD.

Variance reduction (VR) methods \cite{johnson2013accelerating,
defazio2014saga, reddi2015variance, roux2012stochastic, defazio2014finito,
konevcny2014ms2gd, konevcny2013semi, xiao2014proximal, wang2013variance, harikandeh2015stopwasting} reduce the variance in the stochastic gradient estimates, and are able to maintain a
large constant step size to achieve fast convergence to high accuracy.
VR methods exploit the fact that gradient errors are highly correlated
between different uses of the same function $f_{i_k}.$  This is done by subtracting an error
correction term from $\nabla f_{i_k} (x^k)$ that estimates the gradient error from
the most recent use of $f_{i_k}$. Thus the stochastic gradients used by VR
methods have the form
\begin{align}
g^k=  \underbrace{\nabla f_{i_k}(x^k)}_{\text{approximate gradient}} -
\underbrace{\nabla f_{i_k}(y) + \widetilde {g_y}}_{\text{error correction
term}}, \label{vr}
\end{align}
where $y$ is an old iterate, and $\widetilde g_y$ is an approximation of the true gradient $\nabla f(y).$ Typically, $\tilde g_y$ can be kept fixed over an epoch or can be updated cheaply on every iteration.  As an example, the SVRG algorithm \cite{johnson2013accelerating} has an update rule of the form
\begin{align}
x^{k+1} = x^k - \eta \Big(\nabla f_{i_k}(x^k) - \nabla f_{i_k}(y) + \nabla f(y) \Big), \label{svrg}
\end{align}
where $y$ is chosen to be a recent iterate from the algorithm history and is fixed over 1 or 2 epochs, and $\tilde g_y = \nabla f(y)$ is the true gradient of $f$ at $y$, which needs to be computed once every 1 or 2 epochs. Another popular VR algorithm, SAGA  \cite{defazio2014saga},  uses the following corrected gradient approximation
\begin{align}
g^k =   \nabla f_{i_k}(x^k) - \nabla f_{i_k}(\phi_{i_k}) + \frac{1}{n} \sum_{j=1}^n \nabla f_j (\phi_j), \label{saga}
\end{align}
where each $\nabla f_j (\phi_j)$ denotes the most recent value of $\nabla f_j$
and $\phi_j$ denotes the iterate at which the most recent $\nabla f_j$ was
evaluated.  In this case $\widetilde g_y$ is the average of the
$\nabla f_j (\phi_j)$ values for all $i\in \{1,2,\cdots,n\}.$ This error correction term reduces the variance in the stochastic gradients, and thus ensures fast convergence. Notice that for both the algorithms, SVRG and SAGA, if $i_k$ is chosen uniformly at random from $\{1,2,\cdots,n\}$, and conditioning on all $x$, we have $\mathbb{E} \big[g^k \big] = \nabla f (x^k)$. Thus, the error correction term has expected value 0 and the approximate gradient $g^k$ is unbiased for both SVRG and SAGA.

Most work on VR methods has focused on studying their faster convergence rates and better stability properties when compared to classical SGD in the sequential setting. While there have been a few recent papers on parallelizing VR methods, these methods scale poorly in distributed settings and all prior work that we know of has focussed on small-scale parallel or shared memory settings, with the data distributed over 10 or 20 nodes \cite{reddi2015variance, mania2015perturbed, pan2016cyclades}. These parallel algorithms use a parameter server model of computation, and are based on the assumption that communication costs are low, which may not be true in large-scale heterogenous distributed computing environments. The fact that the error correction term reduces the variance in the stochastic gradients, however, seems to indicate that distributed VR methods could be helpful in distributed settings. In particular, the variance-reduced gradients would help in dealing with the problems of instability and slower convergence faced by regular stochastic methods when the frequency of communication between the server and the local nodes is increased.

\subsection{Contributions}
In this work, we use variance reduction to dramatically boost the performance
of SGD in the highly distributed setting. We do this by exploiting the
dependence of VR methods on the gradient correction term  $\widetilde {g_y}.$
We allow many local worker nodes to run simultaneously, while communicating with the
central server only through the exchange of this central error correction
term and the locally stored iterates.  The proposed schemes allow many asynchronous processes to work towards
a central solution with minimal communication, while simultaneously
benefitting from the fast convergence provided by VR.

This work has four main contributions:
\begin{itemize}
 \item First, we present a new VR
algorithm CentralVR, built on SAGA, specifically designed such that it can be easily distributed. We then propose synchronous (CentralVR-Sync) and asynchronous (CentralVR-Async) variations of CentralVR which can linearly scale up over massive datasets using hundreds of cores, and are robust to communication delays between the server and the worker nodes.

\item Second, we theoretically study the convergence of CentralVR when $\widetilde
{g_y}$ is only updated periodically (as in the distributed setting), and
prove linear convergence of the method with constant step sizes. 


\item Third, we propose distributed versions of the existing popular VR algorithms, SVRG and SAGA, that are robust to high communication latency between the worker nodes and the central server, and can scale over large distributed settings ranging over hundreds of nodes. Table \ref{dist_algos} summarizes the distributed algorithms proposed in the paper and their storage and computation requirements.


\item Finally, we present empirical results over different models and datasets that show that these distributed
algorithms can be trained on massive highly distributed datasets in far less
time than existing state-of-the-art stochastic optimization methods.
Performance of all these distributed methods scales linearly up to hundreds of
workers with low communication frequency.  We show empirically
that the proposed methods converge much faster than competing options.  
\end{itemize}

\begin{table}[h!]
\begin{center}
  \caption{Distributed Algorithms Proposed}
  \label{dist_algos}
    \begin{tabular}{| c | c | c | c |}
    \hline
    Proposed Algorithm & Asynchronous? & Storage (No. of gradients) & Gradients/Iteration \\
    \hline
    CentralVR-Sync & No & $n$ & 1  \\
    \hline
    CentralVR-Async & Yes & $n$ & 1  \\
    \hline
    Distributed SVRG & No & $2$ & 2.5  \\
    \hline
    Distributed SAGA & Yes & $n$ & 1  \\
    \hline
  \end{tabular}
  \end{center}
\end{table}

\section{CentralVR Algorithm: single-worker case}
\label{sec:sequential_algo}
We begin by proposing our new VR scheme, CentralVR, in the single-worker case.  As we will see later, the proposed method has a natural generalization to the distributed setting that has low communication requirements.

\subsection{Algorithm Overview}
Our proposed VR scheme is divided into epochs, with $n$ updates taking place in each epoch. Let the iterates generated in the $m$-th epoch be written as $\{x_m^j\}_{j=1}^n.$ Also let $\tilde{x}_m^l$ denote the iterate at which the $l$-th data index was most recently before the $m+1$-th epoch (i.e., on or before the $m$-th epoch). Then, the update rule for CentralVR is given by:
\begin{align} \label{update}
x\kp_{m+1} = x^k_{m+1} - \eta v^k_{m+1},
\end{align}
where $v^k_{m+1}$ is defined as  
\begin{align}
v^k_{m+1} &:= \nabla f_{i_k} (x^k_{m+1}) - \nabla f_{i_k} (\tilde{x}^{i_k}_m) + \overline{g}_m, \label{v_def} \\
\text{and} \quad \overline{g}_m &:= \frac{1}{n} \sum_{j=1}^n \nabla f_j (\tilde{x}^j_m) \nonumber.
\end{align}
Thus, $\overline{g}_m$ is the average of the gradients of all component functions $\{\nabla f_j\}_{j=1}^n,$ each evaluated at the most recent iterate $\{\tilde{x}^j_m\}_{j=1}^n$ at which the corresponding function was used on or before the $m$-th epoch. These gradients are stored in a table, and the average gradient $\overline{g}_m$ is updated at the end of each epoch, i.e., after every $n$ parameter updates.

Note that if $i_k$ is chosen uniformly at random from the set $\{1, 2, \cdots, n\}$ on each iteration $k$, then, conditioning on all history (all $x$), we have $$\mathbb{E} \big[  \nabla f_{i_k} (\tilde{x}^{i_k}_m) \big ] = \frac{1}{n} \sum_{j=1}^n \nabla f_j (\tilde{x}^j_m) = \overline{g}_m.$$ Thus, the error correction term has expected value 0, and $\mathbb{E} \big[ v^k_{m+1} \big] = \nabla f (x^k_{m+1})$, i.e., the approximate gradient $v^k_{m+1}$ is unbiased.

\subsection{Permutation Sampling}

In practical implementations, it is natural to consider a random permutation of the data indices on every epoch, rather than uniformly choosing a random index on every iteration. Thus, on each epoch, a random permutation of the data indices is chosen and a pass is made over the entire dataset, resulting in $n$ updates, one per data sample.  Permutation sampling often outperforms uniform random sampling empirically \cite{bottou2009curiously, bottou2012stochastic}, although theoretical justification for this is still limited (see \cite{gurbuzbalaban2015random, shamir2016without} for some recent results).

As an alternative to uniform random sampling, CentralVR can leverage random permutations over the data indices.  Let $\pi_m$ denote a random permutation of the data indices $\{1, 2, \cdots, n\}$ for the $m$-th epoch, with $\pi_m^j$ denoting the data index chosen in the $j$-th iteration in the $m$-th epoch. Thus, now $\tilde{x}_m^l$ denotes the iterate corresponding to the point when the $l$-th data index was chosen in the $m$-th epoch. The update rule with the random permutation is given by \eqref{update} and \eqref{v_def}, with $i_k = \pi_{m+1}^k$.

Summing \eqref{update} over all $k = 0, 1, \cdots, n - 1$, we get
\begin{align}
\sum_{k=0}^{n-1}  v^k_{m+1} =& \sum_{k=0}^{n-1} \Big( \nabla f_{\pi^k_{m+1}} (x^k_{m+1}) - \nabla f_{\pi^k_{m+1}} (\tilde{x}^{i_k}_m) + \overline{g}_m \Big) \nonumber \\
=&  \sum_{k=0}^{n-1} \Big( \nabla f_{k} (\tilde x^k_{m+1}) - \nabla f_{k} (\tilde{x}^{k}_m) + \overline{g}_m \Big) \nonumber \\
=& \sum_{k=0}^{n-1} \nabla f_{k} (\tilde x^k_{m+1}). \nonumber
\label{summed_v}
\end{align}

Thus, summing \eqref{update} over all $k = 0, 1, \cdots, n - 1$, using the telescoping sum in $x^k_{m+1}$, and using the convention that $x^0_{m+1} = x^n_{m}$, we get
\begin{equation}
x^0_{m+2} = x^0_{m+1} - \eta \sum_{j=1}^n \nabla f_j \big(\tilde{x}^j_{m+1}\big). \label{update_with_v}
\end{equation}

Equation \eqref{update_with_v} shows the update rule in terms of the iterates at the ends of the epochs. Thus, over an epoch, the average gradient accumulated by  CentralVR  is unbiased and thus is a good estimate of the true gradient. This average gradient term can be accumulated cheaply during an
epoch, without any noticeable overhead.


\subsection{Algorithm Details for  CentralVR}
The detailed steps of CentralVR are listed in Algorithm \ref{alg:sequential}. Note, the stored gradients and the average gradient term $\widetilde{g_y}$ are initialized using a single epoch of ``vanilla'' SGD with no VR correction.

\begin{algorithm}[h]
  \begin{algorithmic}[1]
    \STATE \textbf{parameters} learning rate $\eta$
    \STATE \textbf{initialize} $x$, $\{\nabla f_j (\tilde{x}^j)\}_j$, and $\overline{g}$ using plain SGD
    \WHILE {not converged}
    \STATE $\widetilde{g} \gets 0$
    \STATE set $\pi$: random permutation of indices $1, 2, \cdots, n$
    \FOR {$k$ in $\{1, \dots, n\}$}
      \STATE set: $x^{k+1} \gets x^k - \eta \big(\nabla f_{\pi_k} (x^k) - \nabla f_{\pi_k} (\tilde{x}^{\pi_k}) + \overline{g}\big)$
    \STATE accumulate average: $\widetilde{g} \gets \widetilde{g} + \nabla f_{\pi_k} (x^k)/n$
    \STATE store gradient: $\nabla f_{\pi_k} (\tilde{x}^{\pi_k}) \gets \nabla f_{\pi_k} (x^k)$
    \ENDFOR
  \STATE set average gradient for next epoch: $\overline{g} \gets \widetilde{g}$
    \ENDWHILE
  \end{algorithmic} 
  \caption{CentralVR Algorithm: single worker case}  
  \label{alg:sequential}
\end{algorithm}

CentralVR builds on the SAGA method.   SAGA relies on the update
rule \eqref{saga}, which requires an average over a large
number of iterates $(\widetilde {g_y} = \frac{1}{n}\sum_i \nabla f_i(\phi_i))$
to be continuously updated on every iteration.  In the distributed setting,
where the vector $\widetilde {g_y}$ must be shared across nodes, maintaining
an up-to-date average requires large amounts of communication. This makes SAGA less stable in distributed implementations when the communication frequency is decreased. Updating $\widetilde {g_y}$ only occasionally (as we do in the distributed variants of {\em CentralVR} below) translates into significant communication savings in the distributed setting.

CentralVR has the same time and space complexities as SAGA. Namely, on every iteration, 1 gradient computation is required, similar to SGD, and the $n$ gradients $\{ \nabla f_k (\tilde x^k_m) \}_{k=1}^n$ also need to be stored. Note that this is not a significant storage requirement, since for models like logistic regression and ridge regression only a single number is required to be stored corresponding to each gradient.

\section{Convergence Analysis}
\label{sec:sequential_theory}

We now present convergence bounds for Algorithm \ref{alg:sequential}.  We make
the following standard assumptions about the function when studying
convergence properties. First, each $f_i$ is strongly convex with strong convexity
constant $\mu$:
\eqn{eq:strong_convexity_basic}{
f_i(x) \geq f_i(y)  + \nabla f_i(y)^T (x - y)+ \frac{\mu}{2} \| x - y\|^2 .
}
Second, each $f_i$ has Lipschitz continuous gradients with Lipschitz constant $L$ so that
\aln{ 
&f_i(x) \leq f_i(y)+ \nabla f_i(y)^T (x - y)  + \frac{L}{2} \| x - y\|^2.  \label{eq:lipschitz_constant_basic} 
}

We now present our main result. The proof is presented in the appendix.


\begin{theorem}
\label{random_access_thm}
Consider CentralVR with data index $i_k$ drawn uniformly at random (with replacement) on each iteration $k$. Define $\alpha$ as
$$\alpha := \max \Big( 1 - \eta \mu, \frac{2 L^2 \eta}{\mu (1 - 2 L \eta)} \Big).$$
If the step size $\eta$ is small enough such that $0 < \alpha < 1$, then we have the following bound:
\begin{align*}
\big\|x^0_{m+2} - x\opt \big\|^2  + c \big( \overbar{f ( x_{m+1})}  - f(x\opt) \big) \le \alpha \Big( \big\|x^0_{m+1} - x\opt \big\|^2  + c \big( \overbar{f ( \tilde x_m)}  - f(x\opt) \big) \Big),
\end{align*}
where $c = 2 n \eta (1 - 2 L \eta )$ and we define $\overbar{f(x_{m})} := \frac{1}{n}\sum_{k = 0}^{n-1} f(x^k_{m})$.
In other words, the method converges linearly.
\end{theorem}

\subsubsection*{Remark on Step Size Restrictions}
From Theorem \ref{random_access_thm}, notice that CentralVR converges linearly when the step size $\eta$ is small enough such that
\begin{align}
\eta < \min \Big( \frac{1}{\mu}, \frac{1}{2L}, \frac{\mu}{2L(L + \mu)}    \Big). \nonumber
\end{align}
If we observe that $L\ge \mu,$ then we see that this condition is satisfied whenever $\eta < \frac{\mu}{2L(L + \mu)}$.


\section{Distributed Algorithms}
\label{sec:dist_proposed}

We now consider the distributed setting, with a single central server and $p$
local clients, each of which contains a portion of the data set.  In
this setting, the data is decomposed into disjoint subsets $\{\Omega_s\}$, where $s$ denotes a particular local client,
and $\sum_s |\Omega_s| = n.$   We denote the $i$-th function stored on client
$s$ as $f_i^s.$   Our goal is to minimize the global objective function of the form
\begin{equation}
f(x) =  \frac{1}{n} \sum_{s=1}^p  \sum_{j = 1}^{| \Omega_s |} f_j^s(x). \nonumber
\end{equation}

We consider a centralized setting, where the clients can only communicate with the central server, 
and our goal is to derive stochastic algorithms that
scale linearly to high $p$, while remaining stable even under low
communication frequencies between local and central nodes.

\subsection{Synchronous Version}

CentralVR naturally extends to the distributed synchronous setting, and is presented in
Algorithm \ref{alg:sync}. To distinguish the algorithm from the single worker case, we call it CentralVR-Sync.
On each epoch, the local nodes first retrieve a copy of
the central iterate $x,$ and also $\widetilde{g_y}$, which represents the
averaged gradient over all data. The CentralVR method is then performed on each
node, and the most recent gradient for each data point $\nabla f^s_{i_k}
(\tilde{x}^{i_k})$ is stored.  By sharing $\widetilde{g_y}$ across nodes, we
ensure that the local gradient updates utilize global gradient information
from remote nodes.  This prevents the local node from drifting far away from
the global solution, even if each local node runs for one whole epoch before
communicating back with the central server.
  
\begin{algorithm}[h]
  \begin{algorithmic}[1]
    \STATE \textbf{parameters} learning rate $\eta$
    \STATE \textbf{initialize} $x$, $\{\nabla f_j (\tilde{x}^j)\}_j$, $\overline{g}$
    \WHILE {not converged}
    \FOR {each local node $s$}
    \STATE $\widetilde{g} \gets 0$
    \STATE set $\pi$: random permutation of indices $1, 2, \cdots, | \Omega_s |$
    \FOR {$k$ in $\{1, \dots, |\Omega_s|\}$}
      \STATE $x^{k+1} \gets x^k - \eta \big(\nabla f^s_{\pi_k} (x^k) - \nabla f^s_{\pi_k} (\tilde{x}^{\pi_k}) + \overline{g}\big)$
    \STATE accumulate average: $\widetilde{g} \gets \widetilde{g} + \nabla f^s_{\pi_k} (x^k)/|\Omega_s|$
    \STATE store gradient: $\nabla f^s_{\pi_k} (\tilde{x}^{\pi_k}) \gets \nabla f^s_{\pi_k} (x^k)$
    \ENDFOR
    \STATE set average gradient to send to server: $\overline{g} \gets \widetilde{g}$
    \STATE send $x$, $\overline{g}$ to central node
     \STATE receive updated $x$, $\overline{g}$ from central node
    \ENDFOR
    \STATE \textbf{central node:}
    \INDSTATE average $x$, $\overline{g}$ received from workers
    \INDSTATE broadcast averaged $x$, $\overline{g}$ to local workers
    \ENDWHILE
  \end{algorithmic} 
  \caption{CentralVR-Sync Algorithm}  
  \label{alg:sync}
\end{algorithm}
 
%
In CentralVR-Sync, each local node performs local updates for one epoch, or $|\Omega_s|$ iterations, before communicating with the server. 
This is a rather low communication frequency compared to a parameter server
model of computation in which updates are continuously streamed to the central
node. This makes a significant difference in runtimes when the number of local nodes is large, as shown in later sections.

\subsection{Asynchronous Version}

The synchronous algorithm can be extended very easily to the asynchronous case, CentralVR-Async,
as shown in Algorithm~\ref{alg:async}. In CentralVR-Async, the central server keeps a copy of the current iterate $x$ and average gradient $\overline g$. The key idea for CentralVR-Async is that, once a local node completes an epoch, it sends the
\emph{change} in the local averages, given by $\Delta{x}_s$ and $\Delta \overline{g}_s,$ over the last
epoch to the central server. This change is added to the global $x$ and
$\overline g$ to update the parameters stored on the central server. Thus, when the central server
receives parameters from a local node $s$, the updates it performs have the
form
\begin{align*}
x = x + \frac{1}{p} \Delta{x}_s \hspace{2mm} \text{ and }  \hspace{2mm} \overline{g} = \overline{g} + \frac{1}{p}  \Delta \overline{g}_s,
\end{align*}
where $\Delta{x}_s$ and $\Delta \overline{g}_s$ are given by
\begin{align*}
\Delta{x}_s &= \{x^n_{m+1} - x^n_{m}\}_s \hspace{4mm} \text{ and } \\
\Delta \overline{g}_s &= \bigg\{\frac{1}{|\Omega_s|}  \sum_{j\in \Omega_s} \nabla f^s_j (\tilde{x}^j_{m+1}) - \frac{1}{|\Omega_s|}  \sum_{j\in \Omega_s} \nabla f^s_j (\tilde{x}^j_{m})\bigg\}_s.
\end{align*}
 Sending the change in the local parameter values rather than the local parameters themselves ensures that, when updating the central parameter, the previous contribution to the average from that local worker is just replaced by the new contribution. Thus, a fast working local node does not bias the global average solution toward its local solution with an excessive number of updates. This makes the algorithm more robust to heterogenous computing environments where nodes work at disparate speeds.

\begin{algorithm}[h]
  \begin{algorithmic}[1]
    \STATE \textbf{parameters} learning rate $\eta$
    \STATE \textbf{initialize} $x$, $\{\nabla f_j (\tilde{x}^j)\}_j, \overline{g}, \alpha = 1/p, x_\text{old} = \overline{g}_\text{old} = 0$
    \WHILE {not converged}
    \FOR {each local node}
    \STATE $\widetilde{g} \gets 0$
    \STATE set $\pi$: random permutation of indices $1, 2, \cdots, | \Omega_s |$
    \FOR {$k$ in $\{1, \dots, |\Omega_s|\}$}
      \STATE $x^{k+1} \gets x^k - \eta \big(\nabla f^s_{\pi_k} (x^k) - \nabla f^s_{\pi_k} (\tilde{x}^{\pi_k}) + \overline{g}\big)$
    \STATE accumulate average: $\widetilde{g} \gets \widetilde{g} + \nabla f^s_{\pi_k} (x^k)/|\Omega_s|$
    \STATE store gradient: $\nabla f^s_{\pi_k} (\tilde{x}^{\pi_k}) \gets \nabla f^s_{\pi_k} (x^k)$
    \ENDFOR
    \STATE set average gradient: $\overline{g} \gets \widetilde{g}$
    \STATE compute change: $\Delta{x} \gets x - x_\text{old}$, $\Delta{\overline{g}} \gets \overline{g} - \overline{g}_\text{old}$
    \STATE set: $x_\text{old} \gets x$, $\overline{g}_\text{old} \gets \overline{g}$
    \STATE send $\Delta{x}$, $\Delta{\overline{g}}$ to central node
     \STATE receive updated $x$, $\overline{g}$ from central node
    \ENDFOR
    \STATE \textbf{central node:}
    \INDSTATE receive $\Delta{x}$, $\Delta{\overline{g}}$ from a local worker
    \INDSTATE update: $x \gets x + \alpha \Delta{x}$, $\overline{g} \gets \overline{g} + \alpha \Delta{\overline{g}}$
    \INDSTATE send new $x$, $\overline{g}$ back to local worker
    \ENDWHILE
  \end{algorithmic} 
  \caption{CentralVR-Async Algorithm}  
  \label{alg:async}
\end{algorithm}

The proposed CentralVR scheme has several advantages.  It does not require a full gradient computation as in SVRG, and thus can be made fully asynchronous. Moreover, since the average gradient $\widetilde{g_y}$ in the error correction term is updated only at the end of an epoch, communication periods can be increased between the central server and the local nodes, while still maintaining fast and stable convergence.

%
%

\section{Distributed Variants of Other VR Methods}
\label{sec:dist_vr}
In this section, we propose distributed variants of popular variance reduction methods: SVRG and SAGA.  The properties of these variants are overviewed in Table \ref{dist_algos}.  

\subsection{Distributed SVRG}
\label{sec:dist_svrg}

In this section, we present a distributed version of SVRG appropriate for
distributed scenarios with high
communication delays. Recently, in \cite{reddi2015variance}, the authors
presented an asynchronous distributed version of SVRG using a parameter server
model of computation. In SVRG, the average gradient term is $\widetilde {g_y}
= \nabla f(y)$ as shown in \eqref{svrg}. This correction term is very accurate because it uses the entire dataset.  This would indicate that the algorithm would be robust to high
communication periods between the local nodes and the server.

However, a truly asynchronous method is not possible with SVRG since a synchronization step is unavoidable when computing the full gradient. Thus, in this section, we present a synchronous variant of SVRG in Algorithm \ref{alg:sync_svrg}. We define an additional parameter $\tau$ to denote the communication period, i.e., the number of updates to run on each local node before communicating with the central server. 

The true gradient $\overline{g}$ is maintained across all nodes throughout the whole communication period $\tau$, thus ensuring that the local workers stay close to the desired solution, even when $\tau$ is large. After $\tau$ updates, the current iterate ${x}_s$ on each local node $s$ is averaged on the central server to get ${x}$. The true gradient is evaluated at ${x}$, i.e., $\overline{g} = \nabla f({x})$, and $\overline{x} = {x}$ is used on each local node during the next epoch.

\begin{algorithm}[h]
  \begin{algorithmic}[1]
    \STATE \textbf{parameters} step size $\eta$, communication period $\tau$
    \STATE \textbf{initialize} $x$
    \WHILE {not converged}
    \STATE set: $\overline{x} \gets x$
    \STATE  set: $\overline{g} \gets \nabla f(\overline{x})$ via synchronization step
    \FOR {each local node $s$}
    \FOR {$k$ in $\{1, \dots, \tau\}$}
    \STATE sample $i_k \in \{1, \dots, |\Omega_s|\}$ with replacement
      \STATE $x^{k+1} \gets x^k - \eta \big(\nabla f^s_{i_k} (x^k) - \nabla f^s_{i_k} (\overline{x}) + \overline{g}\big)$
    \ENDFOR
    \STATE send $x$ to central node
     \STATE receive updated $x$ from central node
    \ENDFOR
    \STATE \textbf{central node:}
    \INDSTATE average $x$ received from workers
    \INDSTATE broadcast averaged $x$ to local workers
    \ENDWHILE
  \end{algorithmic} 
  \caption{Synchronous SVRG}  
  \label{alg:sync_svrg}
\end{algorithm}

\subsection{Distributed SAGA}
\label{sec:dist_saga}
The update rule for SAGA is given in \eqref{saga}. Since there is no synchronization step required as in SVRG, there is a very natural asynchronous version of the algorithm under the parameter server model of computation.  A linear convergence proof has been presented for the parameter server model of SAGA (see Theorem 3 in \cite{reddi2015variance}). However, this work does not contain any empirical studies of the method. The parameter server framework is a very natural generalization of SAGA, however it has very high bandwidth requirements for large numbers of nodes.

Algorithm \ref{alg:async_saga} presents an asynchronous version of SAGA with lower communication frequency. Like SVRG, we define a communication period parameter $\tau$ which determines the number of iterations to run on each machine before central communication.

\begin{algorithm}[h]
  \begin{algorithmic}[1]
    \STATE \textbf{parameters} step size $\eta$, communication period $\tau$
    \STATE \textbf{initialize} $x$, $\{\nabla f_j (\tilde{x}^j)\}_j, \alpha = 1/p, x_\text{old} = \overline{g}_\text{old} = 0$
    \STATE set average gradient: $\overline{g} \gets \frac{1}{n} \sum_j \nabla f_j (\tilde{x}^j)$
    \WHILE {not converged}
    \FOR {each local node}
    \FOR {$k$ in $\{1, \dots, \tau\}$}
    \STATE sample $i_k \in \{1, \dots, n\}$ with replacement
      \STATE $x^{k+1} \gets x^k - \eta \big(\nabla f^s_{i_k} (x^k) - \nabla f^s_{i_k} (\tilde{x}^{i_k}) + \overline{g} \big)$
      \STATE update: $\overline{g} \gets \overline{g} + \frac{1}{n} \big( \nabla f^s_{i_k} (x^k) -  \nabla f^s_{i_k} (\tilde{x}^{i_k}) \big)$
     \STATE store gradient: $\nabla f^s_{i_k} (\tilde{x}^{i_k}) \gets \nabla f^s_{i_k} (x^k)$
    \ENDFOR
    \STATE compute change: $\Delta{x} \gets x - x_\text{old}$, $\Delta{\overline{g}} \gets \overline{g} - \overline{g}_\text{old}$
    \STATE set: $x_\text{old} \gets x$, $\overline{g}_\text{old} \gets \overline{g}$
    \STATE send $\Delta{x}$, $\Delta{\overline{g}}$ to central node
     \STATE receive updated $x$, $\overline{g}$ from central node
    \ENDFOR
    \STATE \textbf{central node:}
    \INDSTATE receive $\Delta{x}$, $\Delta{\overline{g}}$ from a local worker
    \INDSTATE update: $x \gets x + \alpha \Delta{x}$, $\overline{g} \gets \overline{g} + \alpha \Delta{\overline{g}}$
    \INDSTATE send new $x$, $\overline{g}$ back to local worker
    \ENDWHILE
  \end{algorithmic} 
  \caption{Asynchronous SAGA}  
  \label{alg:async_saga}
\end{algorithm}
In the SAGA algorithm, the average gradient term $\overline{g}$ is updated on each iteration. Thus, as local iterations progress, the average gradient evolves differently on each local node. This makes the algorithm less robust to higher communication periods $\tau$. As the communication period increases, the local nodes drift farther apart from each other and the global solution. Thus, the learning rate needs to shrink as $\tau$ increases over a certain limit. This in turn slows down convergence. For this reason, distributed SAGA is less tolerant to long communication periods than the Algorithms  in Sections \ref{sec:dist_proposed} and \ref{sec:dist_svrg}. However, it still has fast convergence for much higher communication periods than existing stochastic schemes.

The asynchronous SAGA method (Algorithm \ref{alg:async_saga}) is built on the same idea as the proposed asynchronous algorithm: running averages are kept on each local node, and at the end of an epoch the {\em change} in the parameter values are sent to the central server. This makes the algorithm more robust when local nodes work at heterogenous speeds.

In our distributed SAGA algorithm, care has to be taken while updating the average gradient $\overline{g}$. Note that $\overline{g}$  is averaged over the whole dataset. Thus, when replacing the gradient value at the current index $i_k$, the update is scaled down by a factor of $n$ (the total number of global samples, as opposed to $|\Omega_s|,$ the number of local samples). At the end of a local epoch, the average of the stored gradients on each local node is sent back to the central server, along with the current estimate $x$. This ensures that the average gradient term on the central server $\hat{g}$ is built from the most recent gradient computations at each index.

\section{Empirical Results}
\label{sec:results}
In this section, we present the empirical performance of the proposed methods, both in sequential and distributed settings.
We benchmark the methods for two test problems: first, a binary classification problem with $\ell_2$-regularized logistic regression where each $f_i$ is of the form 
$$f_i(x)=\log \big(1+ \exp(b_ia_i^Tx) \big) + \lambda \|x\|^2,$$ 
where feature vector $a_i \in \mathbb{R}^d$ has label $b_i \in \mathbb{R}$.
 We also consider a ridge regression problem of the form $$f_i(x)=(a_i^Tx-b_i)^2 + \lambda \|x\|^2.$$ We present all our results with the $\ell_2$ regularization parameter set at $\lambda = 10^{-4}$, though we found that our results were not sensitive to this choice of parameter.


\begin{figure*}[t]
  \centering
  \begin{subfigure}[t]{0.4\textwidth}
    \includegraphics[width=\textwidth]{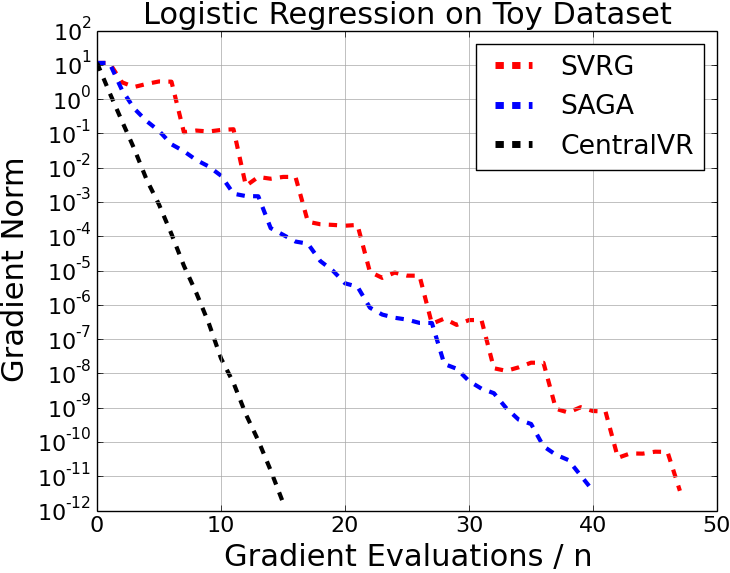}
  \end{subfigure}
  \begin{subfigure}[t]{0.4\textwidth}
    \includegraphics[width=\textwidth]{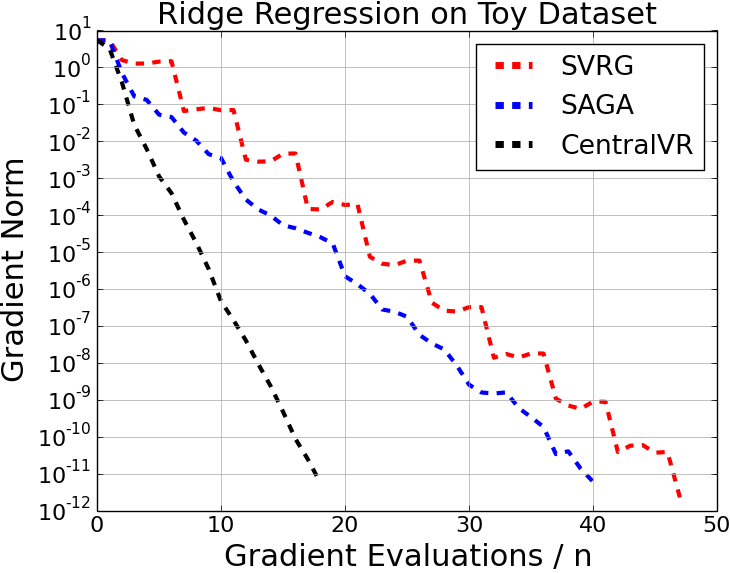}
  \end{subfigure}
  \begin{subfigure}[t]{0.4\textwidth}
    \includegraphics[width=\textwidth]{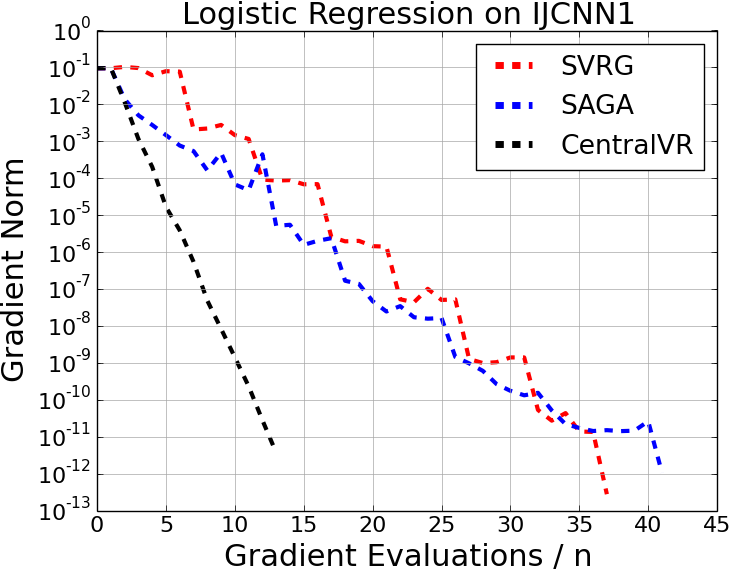}
  \end{subfigure}
  \begin{subfigure}[t]{0.4\textwidth}
    \includegraphics[width=\textwidth]{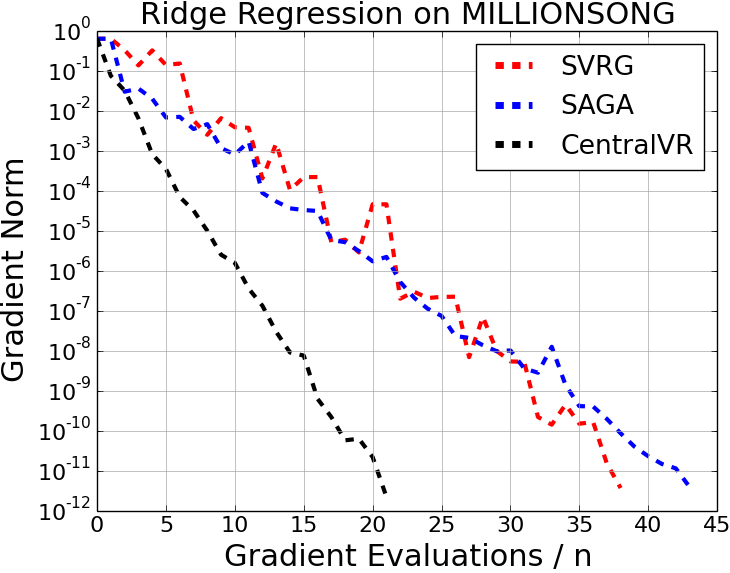}
  \end{subfigure}
  \caption{ Single Worker Results. Logistic regression on toy dataset; Ridge regression on toy data; Logistic regression on IJCNN1 dataset; Ridge regression on MILLIONSONG dataset; In each case CentralVR converges much faster than SVRG and SAGA.}
  \label{fig:seq}
\end{figure*}
 
\begin{figure*}[t]
  \centering
  \begin{subfigure}[t]{0.4\textwidth}
    \includegraphics[width=\textwidth]{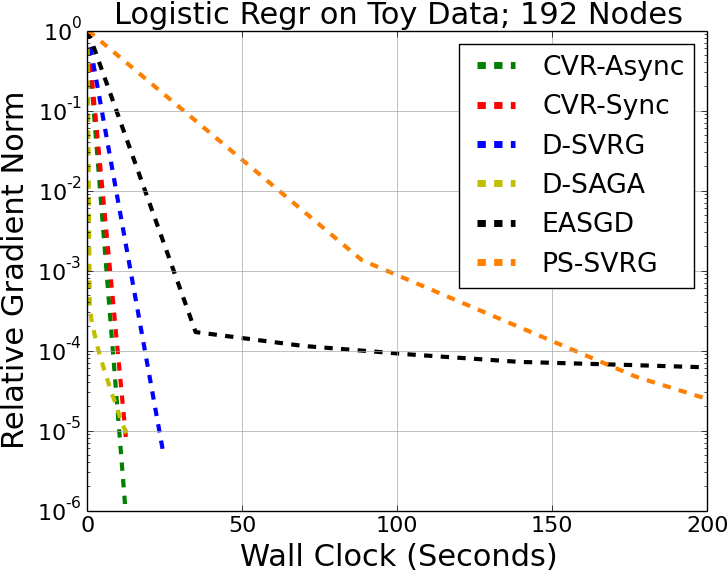}
  \end{subfigure}
  \begin{subfigure}[t]{0.4\textwidth}
    \includegraphics[width=\textwidth]{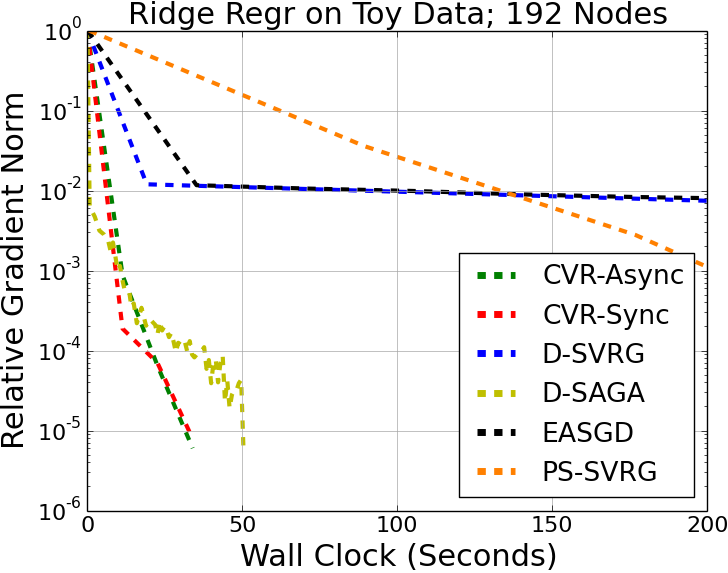}
  \end{subfigure}
  \begin{subfigure}[t]{0.4\textwidth}
    \includegraphics[width=\textwidth]{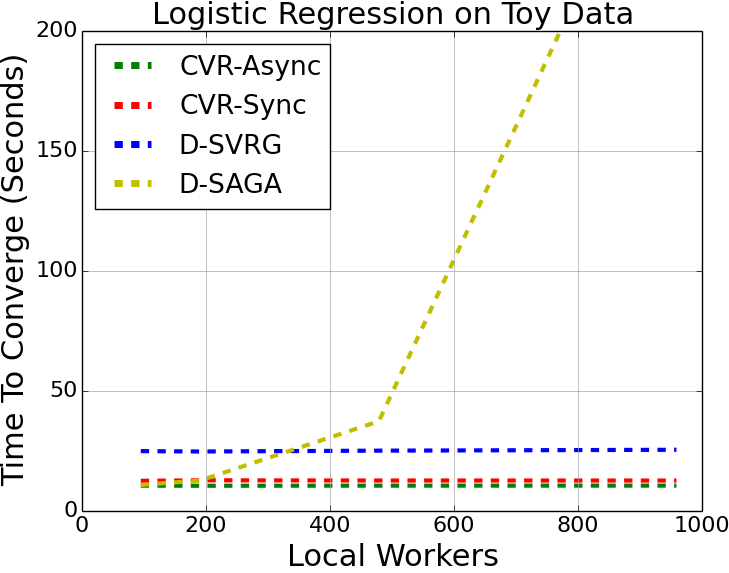}
  \end{subfigure}
  \begin{subfigure}[t]{0.4\textwidth}
    \includegraphics[width=\textwidth]{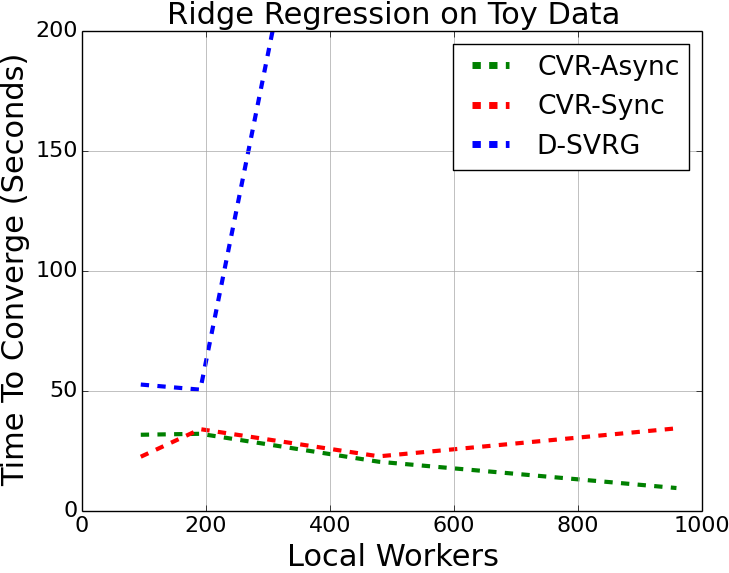}
  \end{subfigure}
  \caption{ Distributed Results on toy datasets for CentralVR-Sync and CentralVR-Async, compared to Distributed SVRG (Section \ref{sec:dist_svrg}), Distributed SAGA (Section \ref{sec:dist_saga}), Parameter Server SVRG and EASGD. Left two plots: Convergence curve for Logistic and ridge regression on synthetic data over 192 nodes. Right two plots: Time required for convergence as number of local workers is increased (data on each local worker is \emph{constant} -- i.e., total data scales \emph{linearly} with the number of local workers) for logistic and ridge regression.}
  \label{fig:dist}
\end{figure*}

\begin{figure*}[t]
  \centering
  \begin{subfigure}[t]{0.4\textwidth}
    \includegraphics[width=\textwidth]{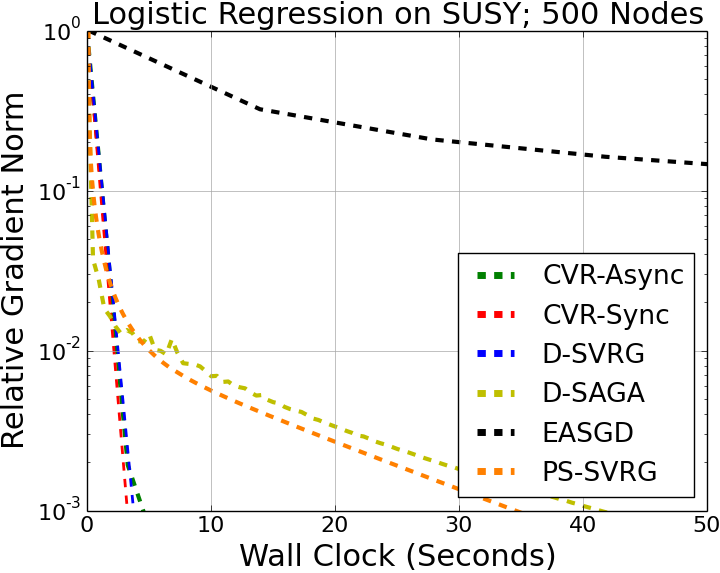}
  \end{subfigure}
  \begin{subfigure}[t]{0.4\textwidth}
    \includegraphics[width=\textwidth]{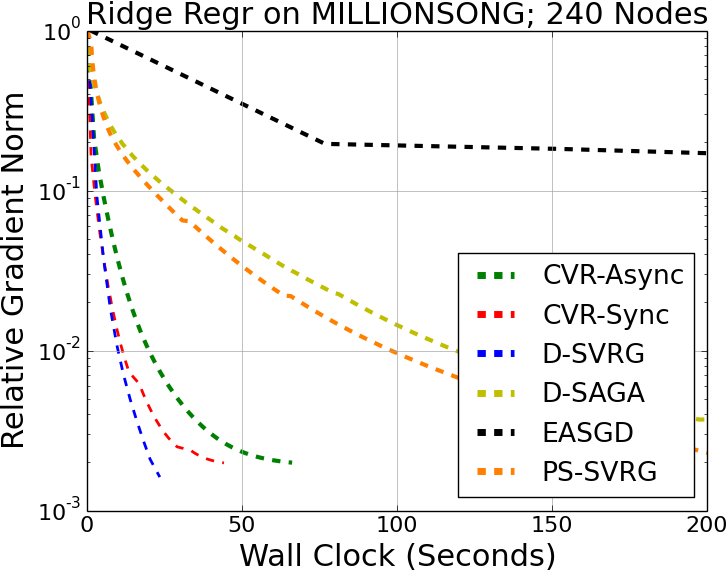}
  \end{subfigure}
  \begin{subfigure}[t]{0.4\textwidth}
    \includegraphics[width=\textwidth]{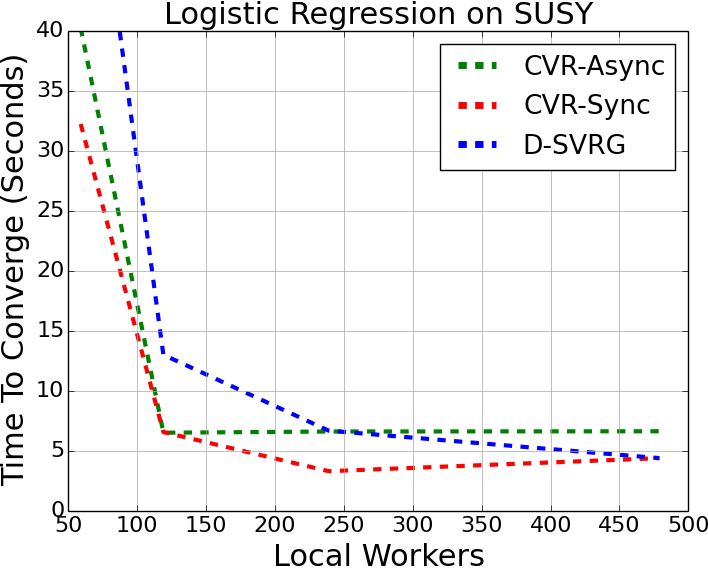}
  \end{subfigure}
  \begin{subfigure}[t]{0.4\textwidth}
    \includegraphics[width=\textwidth]{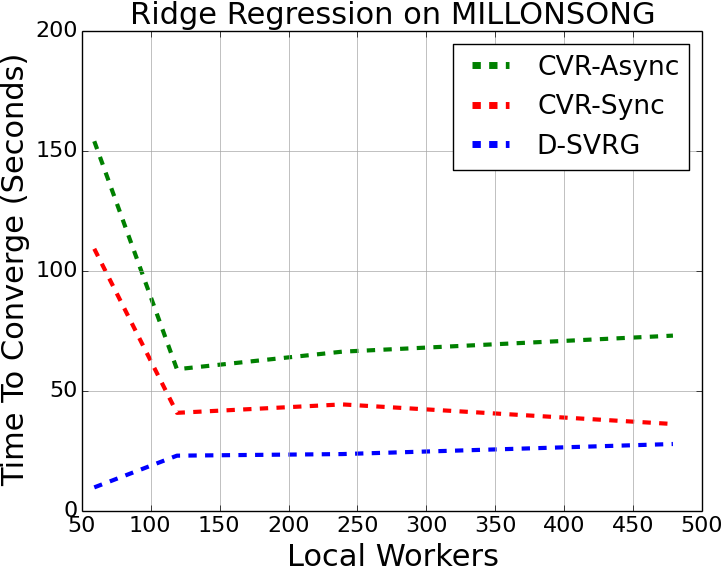}
  \end{subfigure}
  \caption{ Distributed Results on SUSY and MILLIONSONG for CentralVR-Sync and CentralVR-Async, compared to Distributed SVRG (Section \ref{sec:dist_svrg}), Distributed SAGA (Section \ref{sec:dist_saga}), Parameter Server SVRG (Param Server SVRG) and EASGD. (Left two plots) Convergence curve for Logistic regression and ridge regression on SUSY over 500 nodes and on MILLIONSONG over 240 nodes. (Right two plots) Time required for convergence as number of local workers is increased. }
  \label{fig:dist_real}
\end{figure*}

\subsection{Single Worker Results}
\label{sec:seq_results}
We first test our algorithms in the sequential, non-distributed setting. It is well known that VR beats vanilla SGD by a wide margin in many applications.  However, the different VR methods vary widely in their empirical behavior. We compare the single worker CentralVR algorithm to the two most popular VR methods, SVRG \cite{johnson2013accelerating} and SAGA \cite{defazio2014saga}.

We test the methods on two synthetic ``toy'' datasets, in addition to two real-world datasets. Synthetic classification data was generated by sampling two normal distributions with unit variance and means separated by one unit. For the least-squares prediction problem, we generate a random normal matrix $A$ and random labels of the form $b = Ax + \epsilon$, where $\epsilon$ is standard Gaussian noise. For each case, we kept the size of the dataset at $n=5000$ with $d=20$ features. For the binary classification problem, we kept equal numbers of data samples for each class. We also tested performance of our algorithms on two standard real world datasets:  IJCNN1 \cite{prokhorov2001ijcnn} for binary classification and the MILLIONSONG \cite{Bertin-Mahieux2011} dataset for least squares prediction. IJCNN1 contains 35,000 training data samples of 22 dimensions, while MILLIONSONG contains 463,715 training samples of 90 dimensions. For all our experiments, we maintain a constant learning rate, and choose the learning rate that yields fastest convergence.

Results appear in Figure \ref{fig:seq}. We compare convergence rates of the algorithms in terms of number of gradient computations for each method.  This provides a level playing field since different VR methods require different numbers of gradient computations per iteration, and gradient computations dominate the computing time.  The proposed CentralVR algorithm widely out-performs SAGA and SVRG in all cases, requiring less than one-third of the gradient computations of the other methods.

\subsection{Distributed Results}

We now present results of our algorithms in highly distributed settings. We implement the algorithms using a Python binding to MPI, and all experiments were run on an Intel Xeon E5 cluster with 24 cores per node. All our asynchronous implementations are ``locked", where at a given time only one local node can update the parameters on the central server. However, all proposed asynchronous algorithms can be easily implemented in a lock-free setting, leading to further speedups.

We compare the distributed versions of CentralVR, CentralVR-Async [CVR-Async in Figures \ref{fig:dist} and \ref{fig:dist_real}] and CentralVR-Sync [CVR-Sync in Figures \ref{fig:dist} and \ref{fig:dist_real}], proposed in Section \ref{sec:dist_proposed} with the following algorithms:
\begin{itemize}
\item Distributed SVRG (Section \ref{sec:dist_svrg}) [D-SVRG in Figures \ref{fig:dist} and \ref{fig:dist_real}]. We set the communication period $\tau = 2n$ as recommended in \cite{johnson2013accelerating}. We found the performance of the algorithm to be very robust to $\tau$.
\item Distributed SAGA (Section \ref{sec:dist_saga}) [D-SAGA in Figures \ref{fig:dist} and \ref{fig:dist_real}]. We vary the communication period $\tau = \{10, 100, 1000, 10000\}$ and present results for the $\tau$ yielding best results. The algorithm remains relatively stable for $\tau = \{10, 100, 1000\}$ but convergence speeds start slowing down significantly at $\tau = 10000$.
\item Elastic Averaging SGD (EASGD): This is a recently proposed asynchronous SGD method \cite{zhang2015deep} that has been shown to efficiently accelerate training times of deep neural networks. 
As in \cite{zhang2015deep}, we tested the algorithm for communication periods $\tau = \{4, 16, 64\}$, and found results to be nearly insensitive to $\tau$ ($\tau$ updates occur before communication). We also found the regular EASGD algorithm to outperform the momentum version (M-EASGD). We test performance both for a constant step size as well as a decaying step size (using a local clock on each machine) as given by $\eta_0/(1+\gamma k)^{0.5}$  (as in \cite{zhang2015deep}),  where $\eta_0$ is the initial step size, $k$ is the local iteration number, and $\gamma$ is the decay parameter. EASGD has been shown to outperform the related popular asynchronous SGD method Downpour \cite{dean2012large}, on both convex and non-convex settings.
\item Asynchronous ``Parameter Server" SVRG [PS-SVRG in Figures \ref{fig:dist} and \ref{fig:dist_real}]: an asynchronous version of SVRG on a parameter server model of computation \cite{reddi2015variance}. This method outperforms a popular asynchronous SGD method, Hogwild! \cite{recht2011hogwild}, which also uses a parameter server model. We set the epoch size to $2n$, as recommended in \cite{reddi2015variance}. 
\end{itemize}

For the variance reduction methods, we performed experiments using a constant step size, as well as the simple learning rate decay rule $\eta_l = \eta_0\gamma^l$ (here, $l$ is the number of \emph{epochs}, instead of iterations). Decaying the step size does not yield consistent performance gains, and constant step sizes work very well in practice.

We compared the algorithms on a binary classification problem and a least-squares prediction problem using both toy datasets and real world datasets. The toy datasets were created on each local worker exactly the same way as for the sequential experiments. The toy datasets had $d=1000$ features and $|\Omega_s|= 5000$ samples for \emph{each} core $s$, i.e., the total size of the dataset was $p \times 5000$, where $p$ denotes the number of local nodes. We also used the real world datasets MILLIONSONG~\cite{Bertin-Mahieux2011} (containing close to 500,000 data samples) for ridge regression and SUSY \cite{baldi2014searching} (5,000,000 data samples) for logistic regression.

Figure \ref{fig:dist} shows results of our distributed experiments on toy datasets. The left two plots compare the rates of convergence of our algorithms scaled over 192 cores for logistic regression and ridge regression. The $x$-axis displays wall clock time in seconds and the $y$-axis displays the relative norm of the gradient, i.e., the ratio between the current gradient norm and the initial gradient norm. In almost all cases the proposed algorithms, in particular CentralVR, have substantially superior rates of convergence over established schemes. The right two plots in Figure \ref{fig:dist} demonstrate the scalability of our algorithms. On the $y$-axis, we plot the wall clock time (in seconds) required for convergence, and on the $x$-axis, we vary the number of nodes as 96, 192, 480 and 960. Each local worker has $|\Omega_s|=5000$ data points in each case, i.e., the amount of data scales linearly with the number of nodes. Notice that CentralVR-Sync and CentralVR-Async exhibit nearly perfect linear scaling, even when the number of workers is almost 1000. The dataset size in this regime is close to 5 million data points, and the proposed CentralVR methods train both our logistic and ridge regression models to five digits of precision in less than 15 seconds. 

Figure \ref{fig:dist_real} shows results of our distributed experiments on the large datasets SUSY and MILLIONSONG. The left two plots show convergence results for our algorithms over 500 nodes for SUSY and 240 nodes for MILLIONSONG. In both cases, we see that our proposed algorithms outperform or remain competitive with previously proposed schemes. The right two plots show the scaling of our algorithms as we increase the number of local workers for training SUSY and MILLIONSONG. We see that for MILLIONSONG, increasing the number of local workers initially decreases convergence time, but speed levels out for large numbers of workers, likely due to the smaller size of the local dataset fragments. On the larger SUSY problem, we find a consistent decrease in the convergence times as we increase the number of workers. We train on this 5,000,000 sample dataset in less than 5 seconds using 750 local workers.

\section{Conclusion}
This manuscript introduces a new variance reduction scheme, CentralVR, that has lower communication requirements than conventional schemes, allowing it to perform better in highly parallel cloud or cluster computing platforms. In addition, distributed versions of well-known variance reduction stochastic gradient descent (SGD) methods are presented that also perform well in highly distributed settings. We show that by leveraging variance reduction, we can combat the diminishing returns that plague classical SGD methods when scaled across many workers, achieving linear performance scaling to over 1000 cores. This represents a significant increase in scalability over previous stochastic gradient methods.

\section{Acknowledgements}
This work was supported by the National Science Foundation (\#1535902), and the
Office of Naval Research (\#N00014-15-1-2676).

\section{Proofs of Technical Results} \label{sec:proofs}

\subsection{Lemmas}

We first start with two lemmas that will be useful in the proof for Theorem \ref{random_access_thm}.

\begin{lemma} \label{popular_lemma}
For any $f$ defined as $f := \frac{1}{n} \sum_{i=1}^n f_i$, where each $f_i$ satisfies \eqref{eq:strong_convexity_basic} and \eqref{eq:lipschitz_constant_basic}, and on conditioning on any $x$, we have
$$ \mathbb{E} \big\| \nabla f_j(x) -  \nabla f_j(x\opt) \big\|^2  \le 2L (f(x)  - f(x\opt)),  $$
where $j$ is sampled uniformly at random from $\{1, 2, \dots, n\}$ and $x\opt$ denotes the minimizer of $f.$
\end{lemma}

\begin{proof}
A standard result used frequently in the convex optimization literature is as follows:
$$ \| \nabla f_j(x) -  \nabla f_j(x\opt) \|^2 \le 2L ( f_j(x) - f_j(x\opt) - \langle \nabla f_j(x\opt), x - x\opt \rangle ),$$
where $f_j$ is $L$-Lipschitz smooth. A proof for this inequality can be found in \cite{nesterov2013introductory} (Theorem 2.1.5 on page 56).

Since $j$ is sampled uniformly at random from $\{1, 2, \dots, n\}$, we can write
\begin{align*}
\mathbb{E} ( f_j(x) - f_j(x\opt) - \langle \nabla f_j(x\opt), x - x\opt \rangle )  = f(x) - f(x\opt) - \langle \nabla f(x\opt), x - x\opt \rangle = f(x) - f(x\opt),
\end{align*}
where we use the property that $\nabla f(x\opt) = 0$.

Thus, we get the desired result:
$$ \mathbb{E} \big\| \nabla f_j(x) -  \nabla f_j(x\opt) \big\|^2  \le 2L (f(x)  - f(x\opt)). $$
\end{proof}

\begin{lemma} \label{lunch_lemma}
For any $f$ defined as $f := \frac{1}{n} \sum_{i=1}^n f_i$, where each $f_i$ satisfies \eqref{eq:strong_convexity_basic} and \eqref{eq:lipschitz_constant_basic}, and for any $x$ and $i$ we have
$$\big\| \nabla f_i(x) -  \nabla f_i(x\opt) \big\|^2 \le \frac{2L^2}{\mu} \big( f(x) - f(x\opt) \big),$$
where $x\opt$ denotes the minimizer of $f$.
\end{lemma}

\begin{proof}
A standard result used frequently in the convex optimization literature is as follows:
$$ \| \nabla f_i(x) -  \nabla f_i(x\opt) \|^2 \le L^2 \| x - x\opt \|^2,$$
where $f_i$ is $L$-Lipschitz smooth. A proof for this inequality can be found in \cite{nesterov2013introductory} (Theorem 2.1.5 on page 56).

From \eqref{eq:strong_convexity_basic}, we get:
\begin{align}
\| x - x\opt \|^2 \le \frac{2}{\mu} \big( f(x) - f(x\opt) - \langle x-x\opt, \nabla f(x\opt) \rangle    \big) = \frac{2}{\mu} \big( f(x) - f(x\opt) \big),
\end{align}
using the property that $\nabla f(x\opt) = 0$. The desired result follows immediately.
\end{proof}

\subsection{Proof of Theorem \ref{random_access_thm}}

\begin{proof}
Let the update rule for CentralVR be denoted as
$$x^{k+1}_{m+1} = x^{k}_{m+1} - \eta v^{k}_{m+1}, $$
where we define:
$$v^{k}_{m+1} := \Big[  \nabla f_{i_k}(x_{m+1}^k)  - \nabla f_{i_k}( \tilde x_m^{i_k}) + \frac{1}{n}\sum_j \nabla f_j(\tilde x_m^j) \Big].$$

In this proof, we assume that the data indices are accessed randomly with replacement. Thus, $\tilde{x}_m^{i_k}$ denotes the last iterate when the $i_k$-th data index was chosen in or before the $m$-th epoch. Thus, conditioning on all $x$, $v^{k}_{m+1}$ is an unbiased estimator of the true gradient at $x_{m+1}^k$, i.e., we get:
\begin{align}
\eval{v^{k}_{m+1}} =   \nabla f(x_{m+1}^k)  -  \frac{1}{n}\sum_j \nabla f_j ( \tilde x_m^j) + \frac{1}{n}\sum_j \nabla f_j(\tilde x_m^j) =   \nabla f(x^k_{m+1}).
 \label{expected_v}
\end{align}

Conditioned on all history (all $x$), we first begin with the standard identity:
\begin{align}
\mathbb{E} \big[\|x^{k+1}_{m+1} -x\opt\|^2\big] &=  \eval{\|x^{k}_{m+1}- \eta v^{k}_{m+1} -x\opt\|^2}  \nonumber  \\
&= \|x^{k}_{m+1}-x\opt\|^2 - 2\eta   (x^{k}_{m+1} - x\opt)^T \mathbb{E} [v^{k}_{m+1}] +  \eta^2 \mathbb{E} \|  v_{m+1}^k\|^2 \nonumber \\
&= \|x^{k}_{m+1}-x\opt\|^2 - 2\eta   (x^{k}_{m+1} - x\opt)^T \nabla f(x^k_{m+1}) +  \eta^2 \mathbb{E} \|  v_{m+1}^k\|^2,
\label{basic_ineq}
\end{align}
where we use \eqref{expected_v}.

We now bound \eqref{basic_ineq}.
Using the definition of strong convexity in \eqref{eq:strong_convexity_basic}, we can simplify the inner product term in \eqref{basic_ineq} as
\begin{align}
(x\opt - x^{k}_{m+1} )^T \nabla f(x^k_{m+1})  \le  - (f(x^{k}_{m+1}) - f(x\opt) ) - \frac{\mu}{2}\|x\opt - x^{k}_{m+1} \|^2.
\label{prod_ineq}
\end{align}

We now bound the magnitude of the gradient term in \eqref{basic_ineq}:
\begin{align} 
\mathbb{E} \|v^{k}_{m+1}\|^2 =& \mathbb{E} \Big\|  \nabla f_{i_k}(  x_{m+1}^k)  - \nabla f_{i_k}(\tilde x_m^{i_k}) + \frac{1}{n}\sum_j \nabla f_j( \tilde x_m^j) \Big\|^2  \nonumber \\
=& \mathbb{E} \Big\|  \nabla f_{i_k}(  x_{m+1}^k)  - \nabla f_{i_k}(x\opt) + \nabla f_{i_k}(x\opt) - \nabla f_{i_k}(\tilde x_m^{i_k}) + \frac{1}{n}\sum_j \nabla f_j( \tilde x_m^j) \Big\|^2  \nonumber \\
 \le &  2 \mathbb{E} \big\|  \nabla f_{i_k}(x_{m+1}^k)  -   \nabla f_{i_k}(x\opt) \big\|^2 +  2 \mathbb{E} \Big\|  \nabla f_{i_k}( \tilde x_m^{i_k})   -   \nabla f_{i_k}(x\opt) \nonumber \\
& - \Big( \frac{1}{n}\sum_j \nabla f_j(\tilde x_m^j) -  \frac{1}{n}\sum_j \nabla f_j(x\opt) \Big) \Big\|^2 \nonumber \\
  = & 2 \mathbb{E} \big\|  \nabla f_{i_k}(x_{m+1}^k)  -   \nabla f_{i_k}(x\opt) \big\|^2 +  2 \mathbb{E} \Big\|  \nabla f_{i_k}( \tilde x_m^{i_k}) - \nabla f_{i_k}(x\opt) - \mathbb{E} \Big[  \nabla f_{i_k}( \tilde x_m^{i_k}) - \nabla f_{i_k}(x\opt) \Big] \Big\|^2 \nonumber \\
   \le & 2 \mathbb{E} \big\|  \nabla f_{i_k}(x_{m+1}^k)  -   \nabla f_{i_k}(x\opt) \big\|^2  +  2 \mathbb{E} \big\|  \nabla f_{i_k}( \tilde x_m^{i_k}) - \nabla f_{i_k}(x\opt) \big\|^2 \nonumber \\
  \le & 4L \big( f (x_{m+1}^k) - f(x\opt) \big) + \frac{4L^2}{\mu} \mathbb{E} \big( f (\tilde x_m^{i_k})  - f(x\opt) \big).
  \label{norm_grad}
\end{align}
The second equality uses the property that $\nabla f(x\opt) = 0$. The first inequality uses the property that $\| a + b \|^2 \le 2 \|a \|^2 + 2\|b\|^2$. The second inequality uses $\mathbb{E} \| \phi - \mathbb{E} \phi \|^2 = \mathbb{E} \| \phi \|^2 - \| \mathbb{E} \phi \|^2 \le \mathbb{E} \| \phi \|^2$, for any random vector $\phi$. The third inequality follows from Lemma \ref{popular_lemma} and Lemma \ref{lunch_lemma}.

We now plug \eqref{prod_ineq}  and \eqref{norm_grad}  into \eqref{basic_ineq} and rearrange to get 
\begin{align}
&\eval{\|x^{k+1}_{m+1}-x\opt\|^2}    +   2\eta  (1 - 2L \eta  ) \big(  f(x^{k}_{m+1})-f(x\opt) \big) \nonumber  \\
   &\le   \big\|x^{k}_{m+1}-x\opt \big\|^2  -\eta\mu \big\|x^{k}_{m+1}-x\opt \big\|^2  + \frac{4 L^2 \eta^2}{\mu} \mathbb{E} \big( f (\tilde x_m^{i_k})  - f(x\opt) \big).
    \label{aggregated_ineq}
\end{align}

Taking expectation on all $x$ and summing \eqref{aggregated_ineq} over all $k=0, 1, \dots, n-1$, we get a telescoping sum in $\big\|x^{k}_{m+1}-x\opt \big\|^2$ that yields:
\begin{align}
& \mathbb{E} \big\|x^0_{m+2} - x\opt \big\|^2 + 2 n \eta (1 - 2L \eta  ) \mathbb{E} \big( \overbar{f(x_{m+1})} - f(x\opt) \big)   \nonumber \\
   \le & \mathbb{E} \big\|x^0_{m+1} - x\opt \big\|^2    -  \eta \mu \sum_{k=0}^{n-1} \mathbb{E}  \big\|x^{k}_{m+1}-x\opt \big\|^2  + \frac{4 n L^2 \eta^2}{\mu} \mathbb{E} \big( \overbar{f (\tilde x_m)}  - f(x\opt) \big),
  \label{summed_update}
\end{align}
where we use the convention $x^n_{m} = x^0_{m+1}$, and define $\overbar{f(x_{m})}$ as $\overbar{f(x_{m})} := \frac{1}{n}\sum_{k = 0}^{n-1} f(x^k_{m})$.

We now observe that 
$$\mathbb{E}  \big\|x^0_{m+1} - x\opt \big\|^2 \le \sum_{k=0}^{n-1} \mathbb{E}  \big\|x^{k}_{m+1}-x\opt \big\|^2.$$

Thus we can rewrite
$$- \eta \mu \sum_{k=0}^{n-1} \mathbb{E}  \big\|x^{k}_{m+1}-x\opt \big\|^2 \le - \eta \mu \mathbb{E}  \big\|x^0_{m+1} - x\opt \big\|^2.$$

Substituting this in \eqref{summed_update}, we get:
\begin{align*}
& \mathbb{E} \big\|x^0_{m+2} - x\opt \big\|^2 + 2 n \eta (1 - 2 L \eta ) \mathbb{E} \big( \overbar{f(x_{m+1})} - f(x\opt) \big) \\
 & \le (1 - \eta \mu) \mathbb{E} \big\|x^0_{m+1} - x\opt \big\|^2  + \frac{4 n L^2 \eta^2}{\mu} \mathbb{E} \big( \overbar{f (\tilde x_m)}  - f(x\opt) \big).
\end{align*}

We can rewrite this to get:
\begin{align*}
& \mathbb{E} \big\|x^0_{m+2} - x\opt \big\|^2 + 2 n \eta (1 - 2 L \eta ) \mathbb{E} \big( \overbar{f(x_{m+1})} - f(x\opt) \big) \\
&  \le \alpha \Big( \mathbb{E} \big\|x^0_{m+1} - x\opt \big\|^2  + 2 n \eta (1 - 2 L \eta ) \mathbb{E} \big( \overbar{f (\tilde x_m)}  - f(x\opt) \big) \Big),
\nonumber
\end{align*}
where we define $\alpha$ as:
$$\alpha := \max \Big( 1 - \eta \mu, \frac{4 n L^2 \eta^2}{2 n \mu \eta (1 - 2 L \eta )} \Big).$$

The result immediately follows.
\end{proof}

{\small
\bibliography{references}

\begin{thebibliography}{10}

\bibitem{agarwal2011distributed}
Alekh Agarwal and John~C Duchi.
\newblock Distributed delayed stochastic optimization.
\newblock In {\em Advances in Neural Information Processing Systems}, pages
  873--881, 2011.

\bibitem{baldi2014searching}
Pierre Baldi, Peter Sadowski, and Daniel Whiteson.
\newblock Searching for exotic particles in high-energy physics with deep
  learning.
\newblock {\em Nature communications}, 5, 2014.

\bibitem{Bertin-Mahieux2011}
Thierry Bertin-Mahieux, Daniel~P.W. Ellis, Brian Whitman, and Paul Lamere.
\newblock The million song dataset.
\newblock In {\em {Proceedings of the 12th International Conference on Music
  Information Retrieval ({ISMIR} 2011)}}, 2011.

\bibitem{bertsekas1989parallel}
Dimitri~P Bertsekas and John~N Tsitsiklis.
\newblock {\em Parallel and distributed computation: numerical methods}.
\newblock Prentice-Hall, Inc., 1989.

\bibitem{bottou2012stochastic}
L~Bottou.
\newblock Stochastic gradient descent tricks.
\newblock In {\em Neural Networks: Tricks of the Trade}, pages 421--436.
  Springer, 2012.

\bibitem{bottou2009curiously}
L{\'e}on Bottou.
\newblock Curiously fast convergence of some stochastic gradient descent
  algorithms.
\newblock In {\em Proceedings of the symposium on learning and data science,
  Paris}, 2009.

\bibitem{bouchard2015accelerating}
Guillaume Bouchard, Th{\'e}o Trouillon, Julien Perez, and Adrien Gaidon.
\newblock Accelerating stochastic gradient descent via online learning to
  sample.
\newblock {\em arXiv preprint arXiv:1506.09016}, 2015.

\bibitem{byrd2012sample}
Richard~H Byrd, Gillian~M Chin, Jorge Nocedal, and Yuchen Wu.
\newblock Sample size selection in optimization methods for machine learning.
\newblock {\em Mathematical programming}, 134(1):127--155, 2012.

\bibitem{csiba2016importance}
Dominik Csiba and Peter Richt{\'a}rik.
\newblock Importance sampling for minibatches.
\newblock {\em arXiv preprint arXiv:1602.02283}, 2016.

\bibitem{de2017automated}
Soham De, Abhay Yadav, David Jacobs, and Tom Goldstein.
\newblock Automated inference with adaptive batches.
\newblock In {\em International Conference on Artificial Intelligence and
  Statistics}, 2017.

\bibitem{dean2012large}
Jeffrey Dean, Greg Corrado, Rajat Monga, Kai Chen, Matthieu Devin, Mark Mao,
  Andrew Senior, Paul Tucker, Ke~Yang, Quoc~V Le, et~al.
\newblock Large scale distributed deep networks.
\newblock In {\em Advances in Neural Information Processing Systems}, pages
  1223--1231, 2012.

\bibitem{defazio2014saga}
Aaron Defazio, Francis Bach, and Simon Lacoste-Julien.
\newblock Saga: A fast incremental gradient method with support for
  non-strongly convex composite objectives.
\newblock In {\em Advances in Neural Information Processing Systems}, pages
  1646--1654, 2014.

\bibitem{defazio2014finito}
Aaron Defazio, Justin Domke, et~al.
\newblock Finito: A faster, permutable incremental gradient method for big data
  problems.
\newblock In {\em Proceedings of The 31st International Conference on Machine
  Learning}, pages 1125--1133, 2014.

\bibitem{friedlander2012hybrid}
Michael~P Friedlander and Mark Schmidt.
\newblock Hybrid deterministic-stochastic methods for data fitting.
\newblock {\em SIAM Journal on Scientific Computing}, 34(3):A1380--A1405, 2012.

\bibitem{gurbuzbalaban2015random}
Mert G{\"u}rb{\"u}zbalaban, Asu Ozdaglar, and Pablo Parrilo.
\newblock Why random reshuffling beats stochastic gradient descent.
\newblock {\em arXiv preprint arXiv:1510.08560}, 2015.

\bibitem{harikandeh2015stopwasting}
Reza Harikandeh, Mohamed~Osama Ahmed, Alim Virani, Mark Schmidt, Jakub
  Kone{\v{c}}n{\`y}, and Scott Sallinen.
\newblock Stop wasting my gradients: Practical svrg.
\newblock In {\em Advances in Neural Information Processing Systems}, pages
  2242--2250, 2015.

\bibitem{johnson2013accelerating}
Rie Johnson and Tong Zhang.
\newblock Accelerating stochastic gradient descent using predictive variance
  reduction.
\newblock In {\em Advances in Neural Information Processing Systems}, pages
  315--323, 2013.

\bibitem{konevcny2014ms2gd}
Jakub Kone{\v{c}}n{\`y}, Jie Liu, Peter Richt{\'a}rik, and Martin
  Tak{\'a}{\v{c}}.
\newblock ms2gd: Mini-batch semi-stochastic gradient descent in the proximal
  setting.
\newblock {\em arXiv preprint arXiv:1410.4744}, 2014.

\bibitem{konevcny2013semi}
Jakub Kone{\v{c}}n{\`y} and Peter Richt{\'a}rik.
\newblock Semi-stochastic gradient descent methods.
\newblock {\em arXiv preprint arXiv:1312.1666}, 2013.

\bibitem{li2014communication}
Mu~Li, David~G Andersen, Alex~J Smola, and Kai Yu.
\newblock Communication efficient distributed machine learning with the
  parameter server.
\newblock In {\em Advances in Neural Information Processing Systems}, pages
  19--27, 2014.

\bibitem{lian2015asynchronous}
Xiangru Lian, Yijun Huang, Yuncheng Li, and Ji~Liu.
\newblock Asynchronous parallel stochastic gradient for nonconvex optimization.
\newblock In {\em Advances in Neural Information Processing Systems}, pages
  2719--2727, 2015.

\bibitem{mania2015perturbed}
Horia Mania, Xinghao Pan, Dimitris Papailiopoulos, Benjamin Recht, Kannan
  Ramchandran, and Michael~I Jordan.
\newblock Perturbed iterate analysis for asynchronous stochastic optimization.
\newblock {\em arXiv preprint arXiv:1507.06970}, 2015.

\bibitem{mokhtari2016dsa}
Aryan Mokhtari and Alejandro Ribeiro.
\newblock Dsa: Decentralized double stochastic averaging gradient algorithm.
\newblock {\em Journal of Machine Learning Research}, 17(61):1--35, 2016.

\bibitem{needell2014stochastic}
Deanna Needell, Rachel Ward, and Nati Srebro.
\newblock Stochastic gradient descent, weighted sampling, and the randomized
  kaczmarz algorithm.
\newblock In {\em Advances in Neural Information Processing Systems}, pages
  1017--1025, 2014.

\bibitem{nesterov2013introductory}
Yurii Nesterov.
\newblock {\em Introductory lectures on convex optimization: A basic course},
  volume~87.
\newblock Springer Science \& Business Media, 2013.

\bibitem{pan2016cyclades}
Xinghao Pan, Maximilian Lam, Stephen Tu, Dimitris Papailiopoulos, Ce~Zhang,
  Michael~I Jordan, Kannan Ramchandran, Chris Re, and Benjamin Recht.
\newblock Cyclades: Conflict-free asynchronous machine learning.
\newblock {\em arXiv preprint arXiv:1605.09721}, 2016.

\bibitem{prokhorov2001ijcnn}
Danil Prokhorov.
\newblock Ijcnn 2001 neural network competition.
\newblock {\em Slide presentation in IJCNN}, 1, 2001.

\bibitem{recht2011hogwild}
Benjamin Recht, Christopher Re, Stephen Wright, and Feng Niu.
\newblock Hogwild: A lock-free approach to parallelizing stochastic gradient
  descent.
\newblock In {\em Advances in Neural Information Processing Systems}, pages
  693--701, 2011.

\bibitem{reddi2015variance}
Sashank~J Reddi, Ahmed Hefny, Suvrit Sra, Barnab{\'a}s P{\'o}czos, and Alex~J
  Smola.
\newblock On variance reduction in stochastic gradient descent and its
  asynchronous variants.
\newblock In {\em Advances in Neural Information Processing Systems}, pages
  2629--2637, 2015.

\bibitem{robbins1951stochastic}
Herbert Robbins and Sutton Monro.
\newblock A stochastic approximation method.
\newblock {\em The annals of mathematical statistics}, pages 400--407, 1951.

\bibitem{roux2012stochastic}
Nicolas~L Roux, Mark Schmidt, and Francis~R Bach.
\newblock A stochastic gradient method with an exponential convergence \_rate
  for finite training sets.
\newblock In {\em Advances in Neural Information Processing Systems}, pages
  2663--2671, 2012.

\bibitem{shamir2016without}
Ohad Shamir.
\newblock Without-replacement sampling for stochastic gradient methods:
  Convergence results and application to distributed optimization.
\newblock {\em arXiv preprint arXiv:1603.00570}, 2016.

\bibitem{shamir2014distributed}
Ohad Shamir and Nathan Srebro.
\newblock Distributed stochastic optimization and learning.
\newblock In {\em Communication, Control, and Computing (Allerton), 2014 52nd
  Annual Allerton Conference on}, pages 850--857. IEEE, 2014.

\bibitem{wang2013variance}
Chong Wang, Xi~Chen, Alex~J Smola, and Eric~P Xing.
\newblock Variance reduction for stochastic gradient optimization.
\newblock In {\em Advances in Neural Information Processing Systems}, pages
  181--189, 2013.

\bibitem{xiao2014proximal}
Lin Xiao and Tong Zhang.
\newblock A proximal stochastic gradient method with progressive variance
  reduction.
\newblock {\em SIAM Journal on Optimization}, 24(4):2057--2075, 2014.

\bibitem{zhang2015deep}
Sixin Zhang, Anna~E Choromanska, and Yann LeCun.
\newblock Deep learning with elastic averaging sgd.
\newblock In {\em Advances in Neural Information Processing Systems}, pages
  685--693, 2015.

\bibitem{zinkevich2009slow}
Martin Zinkevich, John Langford, and Alex~J Smola.
\newblock Slow learners are fast.
\newblock In {\em Advances in Neural Information Processing Systems}, pages
  2331--2339, 2009.

\bibitem{zinkevich2010parallelized}
Martin Zinkevich, Markus Weimer, Lihong Li, and Alex~J Smola.
\newblock Parallelized stochastic gradient descent.
\newblock In {\em Advances in neural information processing systems}, pages
  2595--2603, 2010.

\end{thebibliography}
\bibliographystyle{plain}
}

\end{document}